\documentclass[10pt]{amsart}

\usepackage{amsmath,amssymb,amsbsy,amsfonts,amsthm,latexsym,
  amsopn,amstext,amsxtra,euscript,amscd,mathrsfs,color,bm}
\usepackage{float}
\usepackage{graphicx}
\usepackage[english]{babel}
\usepackage{mathtools}
\usepackage[hyphens]{url}
\usepackage{setspace}
\usepackage[colorlinks,linkcolor=blue,anchorcolor=blue,citecolor=blue,urlcolor=blue,backref=page]{hyperref}
\usepackage{enumitem}
\usepackage{listings}
\lstdefinestyle{leanstatement}{
  basicstyle=\footnotesize\ttfamily,
  columns=fullflexible,
  keepspaces=true,
  showstringspaces=false,
  breaklines=true,
  breakatwhitespace=true,
  aboveskip=.6\baselineskip,
  belowskip=.6\baselineskip,
  literate=
    {ℕ}{{\ensuremath{\mathbb{N}}}}1
    {ℝ}{{\ensuremath{\mathbb{R}}}}1
    {α}{{\ensuremath{\alpha}}}1
    {φ}{{\ensuremath{\varphi}}}1
    {⇑}{{\ensuremath{\Uparrow}}}1
    {ᵐ}{{\ensuremath{{}^m}}}1
    {γ}{{\ensuremath{\gamma}}}1
    {ε}{{\ensuremath{\epsilon}}}1
    {η}{{\ensuremath{\eta}}}1
    {₀}{{\ensuremath{{}_0}}}1
    {‖}{{\ensuremath{\Vert}}}1
    {⟪}{{\ensuremath{\langle\!\langle}}}1
    {⟫}{{\ensuremath{\rangle\!\rangle}}}1
    {→}{{\ensuremath{\rightarrow}}}1
    {∀}{{\ensuremath{\forall}}}1
    {∃}{{\ensuremath{\exists}}}1
    {∈}{{\ensuremath{\in}}}1
    {∏}{{\ensuremath{\prod}}}1
    {∑}{{\ensuremath{\sum}}}1
    {∧}{{\ensuremath{\wedge}}}1
    {∫}{{\ensuremath{\int}}}1
    {≤}{{\ensuremath{\le}}}1
    {¬}{{\ensuremath{\neg}}}1
    {•}{{\ensuremath{\bullet}}}1
    {⨅}{{\ensuremath{\sqcap}}}1
}

\newif\ifshowfullpaper
\showfullpapertrue
\newsavebox{\hiddenbibliography}
\newcommand{\paperref}[2]{\ifshowfullpaper\ref{#1}\else#2\fi}

\numberwithin{equation}{section}

\newtheorem{theorem}{Theorem}[section]
\newtheorem{lemma}[theorem]{Lemma}
\newtheorem{proposition}[theorem]{Proposition}
\newtheorem{corollary}[theorem]{Corollary}
\theoremstyle{remark}
\newtheorem{remark}[theorem]{Remark}

\newcommand{\R}{\mathbb{R}}
\newcommand{\N}{\mathbb{N}}
\newcommand{\E}{\mathbb{E}}

\newcommand{\Sd}{S^{d-1}}
\newcommand{\Rad}{\mathrm{Rad}}
\newcommand{\Unif}{\mathrm{Unif}}
\newcommand{\ip}[2]{\left\langle #1,#2\right\rangle}
\newcommand{\norm}[1]{\left\lVert #1\right\rVert}

\definecolor{editorange}{RGB}{230,115,0}

\title[]{Depth Separation for Neural Networks under Gaussian Measure}

\author{Cynthia Bortolotto}
\author{Francisco Ganacim}
\author{Roberto I. Oliveira}
\date{}

\begin{document}
\begin{abstract}

We revisit the problem of depth separation for neural networks, as
studied by Eldan and Shamir, Daniely, and others. Specifically, we
consider the natural setting of $L^2$ approximation of radial functions
under the standard Gaussian measure in $d$ dimensions.
We show that there is a radial function in $L^2$ under the standard Gaussian measure 
that can be approximated to a fixed sufficiently small $L^2$ error by
a depth-3 ReLU neural network of polynomial size in $d$ but can only be approximated
by depth-2 networks with subexponentially bounded weights if their
size is at least exponential in $d$.
For that, we prove a general result showing that if a radial function has
substantial mass in the high frequencies of its Hermite expansion, then it cannot be well
approximated by a depth-2 neural network with small width and
bounded weights.
All theorems, lemmas, propositions, and corollaries are verified in Lean~4.
\end{abstract}
\maketitle

\section{Introduction}

Deep neural networks have become a central tool in machine learning. 
Their success in tasks such as image recognition and the generation of speech 
has been gaining dramatic attention over the last couple of years, as they have 
been achieving exceptional performance. Many of these 
models are becoming larger and larger, with billions of parameters and, 
alongside their growth in model size, empirical evidence suggests that 
increasing depth can also play an important role in their success.

On the theoretical side, the expressive power of neural networks has been
studied since the beginnings of the field. Cybenko~\cite{Cybenko1989} and,
independently, Funahashi~\cite{Funahashi1989} and Hornik, Stinchcombe, and
White~\cite{HornikStinchcombeWhite1989} proved that neural networks with a single
hidden layer (depth~2) and a continuous sigmoidal activation function can
uniformly approximate any continuous function on a compact domain;
Leshno, Lin, Pinkus, and Schocken~\cite{LeshnoEtAl1993} later showed that
any continuous nonpolynomial activation function suffices.

To better understand the role of depth in neural networks, a long line of
work has studied depth separations: families of functions that small
deep networks represent but that shallow networks cannot approximate without
exponentially many neurons. Barron~\cite{Barron1994} proved quantitative
approximation bounds for shallow networks under 
Fourier-analytic regularity assumptions.
For networks with growing depth, Telgarsky~\cite{Telgarsky2016}
established exponential approximation separations.
Martens, Chattopadhyay, Pitassi, and
Zemel~\cite{MartensEtAl2013} proved a depth separation result for a function on
$\{0,1\}^{2d}$, which can be computed by a depth-3 network of polynomial size in $d$,
but requires exponential size for depth-2 networks with monotone activations
under suitable bounds on neuron contributions and a fixed classification margin.

Three years later,
Eldan and Shamir~\cite{EldanShamir2016} exhibited a function on $\R^d$ 
computed by a depth-3 network of polynomial size in $d$,
whereas achieving a certain approximation accuracy with a depth-2 network
requires exponential width.
The probability measure underlying their $L^2$
norm is tailored to the Fourier-analytic proof: its density is the
square of the Fourier transform of the indicator function of a Euclidean
ball. 

The following year, Daniely~\cite{Daniely2017} gave a proof of a
separation with respect to the uniform product distribution on
$\Sd\times\Sd$. Using tools from
the theory of orthogonal polynomials, he proved that for inner-product functions
$F(\mathbf x,\mathbf x')=g(\ip{\mathbf x}{\mathbf x'})$ on
$\Sd\times\Sd$, polynomial-size depth-2 networks with exponentially bounded
weights cannot approximate $F$ whenever the function $g$ is far from
low-degree polynomials. For rapidly oscillating functions like
$g(t)=\sin(\pi d^3 t)$, $2^{\Omega(d\log d)}$ hidden neurons are required,
while a depth-3 ReLU network of polynomial size attains any constant
accuracy. Further depth-separation results and 
approximation bounds for radial functions were obtained 
in~\cite{SafranShamir2017,SafranEldanShamir2019}.

\medskip

In this work, we establish an exponential depth separation for neural-network
approximation of certain radial functions 
on the unbounded space $\R^d$, equipped with the standard
Gaussian measure $\gamma^d$. Many elements of our proof are already 
contained in~\cite{Daniely2017}. Still, we believe that our setup of
Gaussian measure and radial functions is more natural than what
has been considered in previous works.

We first fix our terminology.
Let $\sigma(t)=\max(t,0)$ be the ReLU activation function. A
depth-2 network of width $w$ with parameters
$\mathbf w_1,\ldots,\mathbf w_w\in\R^d$ and
$v_1,\ldots,v_w,b_1,\ldots,b_w\in\R$ is the function given by
\[
  x
  \;\longmapsto\;
  \sum_{i=1}^{w} v_i\,\sigma\bigl(\ip{\mathbf w_i}{x}+b_i\bigr),
  \qquad x\in\R^d.
\]
Similarly, a depth-3 network of widths $w_1$ and $w_2$, with
parameters $\mathbf w_1,\ldots,\mathbf w_{w_1}\in\R^d$ and $b_j$,
$v_{ij}$, $c_i$, $u_i\in\R$, is the function
\[
  x
  \;\longmapsto\;
  \sum_{i=1}^{w_2} u_i\,
  \sigma\biggl(\sum_{j=1}^{w_1} v_{ij}\,
  \sigma\bigl(\ip{\mathbf w_j}{x}+b_j\bigr)+c_i\biggr),
  \qquad x\in\R^d.
\]
We call $\max\{w_1,w_2\}$ the width of this network. When discussing bounds on parameters
of these networks, we use the Euclidean norm for weight vectors
and the absolute value for scalar weights and biases.

A ridge function on $\R^d$ is a function that depends on its input
through a single one-dimensional projection, that is, a function of the form
\[
  x\longmapsto g\bigl(\ip{v}{x}\bigr),
  \qquad v\in\Sd,
\]
where $g\colon\R\to\R$. These functions
are widely studied in approximation theory.
 Note that a
depth-2 ReLU network of width $w$ is such a sum of $r=w$ ridge functions.
For $\mathbf w_i\ne0$, take the direction
$\mathbf w_i/\norm{\mathbf w_i}$ and the ridge function
$g_i(t)=v_i\,\sigma(\norm{\mathbf w_i}t+b_i)$.
If $\mathbf w_i=0$, then
$g_i(t)=v_i\,\sigma(b_i)$.

Throughout, $\gamma$ denotes the one-dimensional standard Gaussian measure and
$\gamma^d$ the standard Gaussian measure on $\R^d$. 
The multivariate Hermite polynomials are an orthogonal basis of $L^2(\gamma^d)$.
Let $P_n$ denote the orthogonal projection in $L^2(\gamma^d)$ onto the 
span of the multivariate Hermite polynomials of total degree $n$. For
$n_0\in\N$, write $P_{\ge n_0}=\sum_{n\ge n_0}P_n$.
(The definition of multivariate Hermite polynomials is recalled in
Section~\paperref{sec:hermite}{2}.)

Finally, for $V=(V_1,\ldots,V_d)$ uniformly distributed on the sphere $\Sd$
and even $n$, we define
\[
  K_n:=\E[V_1^n]
  =\prod_{i=0}^{n/2-1}\frac{2i+1}{d+2i}.
\]
The product formula is proved in
Proposition~\paperref{prop:max-moment}{3.4}. Note that it yields
$K_n\le(n/(d+n))^{n/2}$ for even $2\le n\le d$,
so $K_n$ is exponentially small
when $n$ is proportional to $d$; in particular, 
$K_d\le2^{-d/2}$ for even $d$.

Our main lower bound states that a radial function whose Hermite
expansion has a substantial component in the high Hermite
degrees cannot be
well approximated by few ridge functions.

\begin{theorem}
\label{thm:ridge-lb}
Let $n_0$ be even, let $F\in L^2(\gamma^d)$ be radial and set
$A=\norm{P_{\ge n_0}F}_{L^2(\gamma^d)}$. Let $M>0$, and let
$g_1,\ldots,g_r\in L^2(\gamma)$ with $\norm{g_j}_{L^2(\gamma)}\le M$ and
$v_1,\ldots,v_r\in\Sd$. If
\[
  \biggl\|F-\sum_{j=1}^r g_j\bigl(\ip{v_j}{\cdot}\bigr)
  \biggr\|_{L^2(\gamma^d)}^2
  \le\frac{A^2}{2},
  \qquad\text{then}\qquad
  r\;\ge\;\frac{A}{4M\sqrt{K_{n_0}}}.
\]
\end{theorem}

For a radial target of the form $F(x)=\phi(\norm{x}^2/d)$, the condition in
Theorem~\ref{thm:ridge-lb} has a natural interpretation in terms of polynomial
approximation. The variable $\norm{X}^2/d$, for $X\sim\gamma^d$, has the
normalized $\chi_d^2$ distribution, with density $\rho_d(t)$. If $\phi$ cannot
be well approximated by low-degree polynomials in $L^2(\rho_d(t)\,dt)$, then
$F$ has substantial Hermite mass in high degrees. The following corollary
therefore gives a width lower bound directly from this polynomial
approximation condition.

\begin{corollary}
\label{cor:polynomial-criterion}
Let $d\ge2$ and let $n_0\ge2$ be even. Let
$F\in L^2(\gamma^d)$ be such that $F(x)=\phi(\norm{x}^2/d)$ for some measurable $\phi$
and suppose that, for some $a>0$,
\[
  \inf_{\substack{q\in\R[t]\\\deg q<n_0/2}}
  \int_0^\infty(\phi(t)-q(t))^2\rho_d(t)\,dt\ge a^2.
\]
Let $M>0$, and let
$g_1,\ldots,g_r\in L^2(\gamma)$ and $v_1,\ldots,v_r\in\Sd$ satisfy
$\norm{g_j}_{L^2(\gamma)}\le M$. If
\[
  \biggl\|F-\sum_{j=1}^r g_j\bigl(\ip{v_j}{\cdot}\bigr)
  \biggr\|_{L^2(\gamma^d)}^2\le\frac{a^2}{2},
\]
then
\[
  r\ge\frac{a}{4M\sqrt{K_{n_0}}}.
\]
\end{corollary}

\begin{proof}
By Lemma~\paperref{lem:radial-poly-distance}{5.1}, applied to
$F(x)=\phi(\norm{x}^2/d)$, the polynomial approximation hypothesis gives
$A:=\norm{P_{\ge n_0}F}_{L^2(\gamma^d)}\ge a$.
The assumed squared error is therefore at most $a^2/2\le A^2/2$.
Theorem~\ref{thm:ridge-lb} yields
$r\ge A/(4M\sqrt{K_{n_0}})\ge a/(4M\sqrt{K_{n_0}})$.
\end{proof}

The main idea behind the proof of Theorem~\ref{thm:ridge-lb} is to exploit the 
interaction between
the Hermite decomposition associated with Gaussian measure and the rotational
symmetry of the target function. 
This is similar to the general strategy of Daniely~\cite{Daniely2017}.
Since the target is radial, only the radial part of a ridge
function contributes to their correlation. We therefore average the ridge
function over rotations and examine its Hermite expansion. In each Hermite
degree, this averaging introduces moments of a uniformly distributed random
vector on $\Sd$ and the decay of these moments is what allows us to bound the
correlation.





Specializing $F$ to a rapidly oscillating
radial sine function and taking $n_0=d$ yields the following depth separation.

\begin{theorem}
\label{thm:depth-sep}
Let $F : \R^d \to \R$ be defined by
\[
  F(x)=\sin\bigl(4\pi d^{1/2}\norm{x}^2\bigr).
\]
Then, there are universal constants $c>0$ and $d_0$ such that for every even
$d\ge d_0$ the following holds:

\begin{enumerate}
\item[\textup{(i)}] Let $M>0$. For all
$g_1,\ldots,g_r\in L^2(\gamma)$ with $\norm{g_j}_{L^2(\gamma)}\le M$ and
all $v_1,\ldots,v_r\in\Sd$, if
\[
  \biggl\|F-\sum_{j=1}^r g_j\bigl(\ip{v_j}{\cdot}\bigr)
  \biggr\|_{L^2(\gamma^d)}^2
  \le\frac{c^2}{2},
  \qquad\text{then}\qquad
  r\;\ge\;\frac{c}{4M}\,2^{d/4}.
\]
\item[\textup{(ii)}] For every $0<\epsilon\le1$, there
is a depth-3 ReLU network $\mathcal N$ such that
\[
  \norm{F-\mathcal N}_{L^2(\gamma^d)}\le\epsilon,
\]
with width 
$O\bigl(d^{5/2}(d+\log(1/\epsilon))/\epsilon\bigr)$, and with parameters
bounded by $O\bigl(d+\log(1/\epsilon)\bigr)$.
\end{enumerate}
The implicit constants are universal, independent of $d$ and $\epsilon$.
\end{theorem}

For the lower bound, we also follow Daniely's
strategy~\cite{Daniely2017}. Write $F(x)=\phi(\norm{x}^2/d)$ with
$\phi(t)=\sin(4\pi d^{3/2}t)$. We exploit the fact that $\phi$ is highly oscillating
to show that
it cannot be well approximated by low-degree polynomials in $L^2(\rho_d(t)\,dt)$. 
Therefore, the first part of Theorem~\ref{thm:depth-sep} 
follows from Lemma~\ref{lem:energy-lb} and Theorem~\ref{thm:ridge-lb}.
For the second part, we construct a depth-2 network that approximates $\norm{x}^2$ and,
using the same approach, a depth-2 network that approximates the sine function. 
Composing these two networks yields a depth-3 network that approximates $F$.

As a straightforward corollary, we get the comparison below that shows that
at any fixed sufficiently small approximation error, depth-2 networks
with the stated parameter bound require exponential size, whereas depth-3
networks attain the same error with polynomial size.

\begin{corollary}
\label{cor:matched-depth-comparison}
There exist constants $\epsilon_0\in(0,1)$, $A\ge1$, and $d_1\ge1$
such that for every even $d\ge d_1$ and every $0<\epsilon\le\epsilon_0$, setting
$ B_{d,\epsilon}:=A\bigl(d+\log(1/\epsilon)\bigr)$
the function
\[
  F_d(x)=\sin\bigl(4\pi\sqrt d\,\norm{x}^2\bigr)
\]
has the following properties:
\begin{enumerate}
\item[\textup{(i)}] Every depth-2 ReLU network $\mathcal N$ whose parameters
are bounded by $B_{d,\epsilon}$ and which satisfies
$\norm{F_d-\mathcal N}_{L^2(\gamma^d)}\le\epsilon$ has width
\[
  \Omega\!\left(\frac{2^{d/4}}{\bigl(d+\log(1/\epsilon)\bigr)^2}\right).
\]
\item[\textup{(ii)}] There exists a depth-3 ReLU network $\mathcal N$ whose
parameters are bounded by $B_{d,\epsilon}$, whose error satisfies
$\norm{F_d-\mathcal N}_{L^2(\gamma^d)}\le\epsilon$, and whose width is
\[
  O\!\left(\frac{d^{5/2}\bigl(d+\log(1/\epsilon)\bigr)}{\epsilon}\right).
\]
\end{enumerate}
\end{corollary}

The remainder of the paper is organized as follows.
In Section~\paperref{sec:hermite}{2}, we prove the necessary facts about
multivariate Hermite polynomials and their behavior under
rotations. In Section~\paperref{sec:radialization}{3}, we study the
radialization of ridge functions and compute the spherical moments that will 
bound these functions. 
In Section~\paperref{sec:ridge-lower-bound}{4} we combine these
results to establish Theorem
 \ref{thm:ridge-lb}.
Finally, in Section~\paperref{sec:depth-separation}{5}, we apply the theorem's lower
bound to establish the first part of Theorem \ref{thm:depth-sep}, and then conclude it
by constructing a polynomial-size
depth-3 ReLU approximation for the target function.
In Section~\paperref{sec:lean-verification}{6}, we discuss the formal
verification of the main results in Lean~4; the code is available in a
\href{https://github.com/cynthiabort/Depth-Separation-for-Neural-Networks-under-Gaussian-Measure---Lean-}{companion GitHub repository}.

\ifshowfullpaper

\section{Hermite Expansions and Rotational Symmetry}
\label{sec:hermite}

Two structures of Gaussian space are central to our argument: the orthogonal
decomposition into Hermite projections and invariance under rotations. In this
section, we present these facts, which are standard in the literature,
and prove them for completeness.

We use the probabilists' Hermite polynomials $h_n$, defined by
\[ h_n(t)=(-1)^n e^{t^2/2}\frac{d^n}{dt^n}e^{-t^2/2},\]
$n \ge 0$. 
They are normalized by
\[
  \ip{h_m}{h_n}_{L^2(\gamma)}=m!\,\delta_{mn}.
\]
For a multi-index $\alpha=(\alpha_1,\ldots,\alpha_d)\in\N^d$, set
\[
  H_\alpha(x)=\prod_{i=1}^d h_{\alpha_i}(x_i),
  \qquad
  \alpha!=\prod_{i=1}^d \alpha_i!,
  \qquad
  |\alpha|=\sum_{i=1}^d \alpha_i.
\]
For $v=(v_1,\ldots,v_d)\in\R^d$, we also write
$v^\alpha=\prod_{i=1}^d v_i^{\alpha_i}$.
Then
\[
  \ip{H_\alpha}{H_\beta}_{L^2(\gamma^d)}
  =
  \alpha!\,\delta_{\alpha\beta},
\]
and the $n$-th Gaussian chaos is
\[
  \mathcal H_n
  =
  \operatorname{span}\{H_\alpha:|\alpha|=n\}.
\]

The first identity relates Hermite polynomials of linear functionals to
the multivariate Hermite basis.

\begin{lemma}
\label{lem:hermite-ridge}
For every $n\in\N$, $v\in\Sd$, and $x\in\R^d$,
\[
  h_n(\ip{v}{x})
  =
  \sum_{|\alpha|=n}\frac{n!}{\alpha!}v^\alpha H_\alpha(x).
\]
\end{lemma}

\begin{proof}
The generating function for the probabilists' Hermite polynomials is
\[
  \exp\left(tz-\frac{t^2}{2}\right)
  =
  \sum_{n=0}^\infty h_n(z)\frac{t^n}{n!}.
\]
Since $\norm{v}=1$, applying the generating-function identity first with
$z=\ip{v}{x}$ and then coordinatewise gives
\[
\begin{aligned}
  \sum_{n=0}^\infty h_n(\ip{v}{x})\frac{t^n}{n!}
  &=\exp\left(t\ip{v}{x}-\frac{t^2}{2}\right)\\
  &=
  \prod_{i=1}^d
  \exp\left(tv_ix_i-\frac{(tv_i)^2}{2}\right)\\
  &=
  \prod_{i=1}^d
  \sum_{\alpha_i=0}^\infty h_{\alpha_i}(x_i)
  \frac{(tv_i)^{\alpha_i}}{\alpha_i!}\\
  &=
  \sum_{\alpha\in\N^d}
  H_\alpha(x)\frac{t^{|\alpha|}v^\alpha}{\alpha!}.
\end{aligned}
\]
Comparing the coefficient of $t^n$ on the two sides proves the lemma.
\end{proof}

We next show that the Gaussian chaos projections commute with orthogonal transformations.

\begin{lemma}
\label{lem:rotation-commutes}
For every $U\in O(d)$ and every $f\in L^2(\gamma^d)$,
\[
  P_n(f\circ U)=(P_n f)\circ U.
\]
\end{lemma}

\begin{proof}
Define the operator $T_U$ on $L^2(\gamma^d)$ by
\[
  T_Uf=f\circ U.
\]
It suffices to prove that $T_U$ preserves both $\mathcal H_n$ and
$\mathcal H_n^\perp$. Indeed, let $f=P_nf+(I-P_n)f.$ Applying $T_U$ gives
\[
  T_Uf=T_U(P_nf)+T_U((I-P_n)f).
\]
On the other hand, we also have
\[
T_Uf = P_n(T_U f) + (I- P_n)(T_U f).
\]
If $T_U$ preserves the two spaces, then it holds that $P_n(T_Uf) = T_U(P_n(f)),$ proving the statement.

We first prove that $T_U$ preserves $\mathcal H_n$. Note that the ridge Hermites
\[
  \bigl\{x\mapsto h_n(\ip{v}{x}):v\in\Sd\bigr\}
\]
span this space. Indeed, by Lemma \ref{lem:hermite-ridge}, these functions
belong to $\mathcal H_n$. Conversely, let
$Q=\sum_{|\alpha|=n}a_\alpha H_\alpha\in\mathcal H_n$ and suppose that $Q$
is orthogonal to all the ridge Hermites. Then, for every $v\in\Sd$,
\[
  0
  =
  \ip{Q}{h_n(\ip{v}{\cdot})}_{L^2(\gamma^d)}
  =
  \sum_{|\alpha|=n}
  a_\alpha\frac{n!}{\alpha!}v^\alpha
  \ip{H_\alpha}{H_\alpha}_{L^2(\gamma^d)}
  =
  n!\sum_{|\alpha|=n}a_\alpha v^\alpha.
\]
Thus the polynomial $\sum_{|\alpha|=n}a_\alpha v^\alpha$
vanishes on the sphere and, since it is homogeneous, it must vanish on all of $\R^d$,
implying that all coefficients $a_\alpha$ are zero.

Therefore, it suffices to understand the action of $T_U$ on the ridge Hermites.
If $v\in\Sd$, then
$U^t v\in\Sd$ and
\[
  T_U\bigl(h_n(\ip{v}{\,\cdot\,})\bigr)(x)
  =
  h_n(\ip{v}{Ux})
  =
  h_n(\ip{U^t v}{x}).
\]
Hence $T_U\mathcal H_n\subseteq\mathcal H_n$, and applying the
same argument to $U^{-1}$ gives the reverse inclusion, proving that
$T_U\mathcal H_n=\mathcal H_n$.

It remains to prove that $T_U$ preserves $\mathcal H_n^\perp$. Since $U$ is
orthogonal, rotation invariance of $\gamma^d$ implies that $T_U$ is a
unitary operator satisfying
$T_U^*=T_{U^{-1}}$. Hence, if
$h\in\mathcal H_n^\perp$ and $Q\in\mathcal H_n$, then
$T_{U^{-1}}Q\in\mathcal H_n$ and
\[
  \ip{T_Uh}{Q}
  =
  \ip{h}{T_{U^{-1}}Q}
  =
  0.
\]
Thus $T_U$ also preserves
$\mathcal H_n^\perp$, completing the proof.
\end{proof}

\section{Radialization of Ridge Functions}
\label{sec:radialization}

In this section, our goal is to understand the behaviour of the radialized ridge functions. We 
recall once more standard facts about these functions, and write the proofs for completeness.
We expand the ridge functions in the Hermite basis, and then average
them over all possible rotations to obtain an expression of their radializations in terms of moments
of a uniform vector on the sphere.

In general, for a function $f\in L^2(\gamma^d)$, we define its radialization $\Pi_{\Rad}f$ by
\[
  \Pi_{\Rad}f(x)=\int_{O(d)} f(Ux)\,dU,
\]
where $dU$ is Haar probability measure on $O(d)$. Note that $\Pi_{\Rad}$
is the orthogonal projection of $L^2(\gamma^d)$ onto the subspace of radial functions.
Moreover,
we prove that $\Pi_{\Rad}$ commutes with the Gaussian chaos projections $P_n$. Indeed,

\begin{proposition}
\label{prop:radialization-commutes}
For every $n\in\N$, it holds that
\[
  P_n\Pi_{\Rad}=\Pi_{\Rad}P_n.
\]
\end{proposition}

\begin{proof}
Let $\mathcal R$ denote the closed subspace of radial functions. If
$f\in\mathcal R$, then for every $U\in O(d)$, the preceding lemma gives
\[
  (P_nf)\circ U=P_n(f\circ U)=P_nf,
\]
so $P_nf\in\mathcal R$. Since $P_n$ is self-adjoint, it also preserves
$\mathcal R^\perp$. Therefore, $P_n$ preserves the orthogonal decomposition
\[
  L^2(\gamma^d)=\mathcal R\oplus\mathcal R^\perp,
\]
proving that it commutes with the orthogonal projection $\Pi_{\Rad}$ onto
$\mathcal R$.
\end{proof}

We can now understand the case of radial projections of ridge functions. Let
$g\in L^2(\gamma)$ and $v\in\Sd$, and define
\[
  g^v(x)=g(\ip{v}{x}).
\]
Observe that
\(
  \norm{g^v}_{L^2(\gamma^d)}=\norm{g}_{L^2(\gamma)},
\)
since $\ip{v}{X}\sim\gamma$ whenever $X\sim\gamma^d$.

In what follows, we obtain an explicit description of the radialization
of ridge functions in terms of the Hermite expansion of $g$ and the moments
of a uniform vector on the sphere.

\begin{proposition}
\label{prop:radialized-ridge}
Let $V\sim\Unif(\Sd)$, $g\in L^2(\gamma)$, and $v\in\Sd$. Then
\[
  \Pi_{\Rad}(g^v)
  =
  \sum_{n=0}^\infty\sum_{|\alpha|=n}
  \frac{\ip{g}{h_n}_{L^2(\gamma)}}{\alpha!}
  \E[V^\alpha]\,H_\alpha.
\]
The identity holds in $L^2(\gamma^d)$, with the series converging in that norm.

\end{proposition}

\begin{proof}
We first prove the property for finite Hermite expansions and then extend it to
the general case by an approximation argument.
Suppose $g$ is a finite Hermite polynomial,
\[
  g=\sum_{n=0}^N a_n h_n,
\]
then
\[
  g^v(x)=\sum_{n=0}^N a_n h_n(\ip{v}{x}).
\]
Let $U\in O(d)$, and note that, by Lemma~\ref{lem:hermite-ridge}, for every
$n\in\N$,
\begin{equation}
\label{eq:rotated-ridge-hermite}
  h_n(\ip{v}{Ux})
  =h_n(\ip{U^tv}{x})
  =\sum_{|\alpha|=n}\frac{n!}{\alpha!}(U^tv)^\alpha H_\alpha(x).
\end{equation}
Since $U^tv$ is uniformly distributed on $\Sd$ when $U$ is Haar-distributed
on $O(d)$, averaging \eqref{eq:rotated-ridge-hermite} over $U$, we get that
\[
  \Pi_{\Rad}(g^v)(x)
  =
  \sum_{n=0}^N a_n
  \sum_{|\alpha|=n}\frac{n!}{\alpha!}\E[V^\alpha]H_\alpha(x).
\]
Note that $a_n=\ip{g}{h_n}_{L^2(\gamma)}/n!$, so the desired formula
for a finite Hermite expansion follows.

Now let $g\in L^2(\gamma)$ be arbitrary, with its Hermite expansion given by
\[
  g=\sum_{n=0}^\infty
  \frac{\ip{g}{h_n}_{L^2(\gamma)}}{n!}\,h_n.
\]
Denote by
\[
  g_N=\sum_{n=0}^N
  \frac{\ip{g}{h_n}_{L^2(\gamma)}}{n!}\,h_n,
\]
the partial sums of the Hermite expansion, which
converge to $g$ in $L^2(\gamma)$. The finite
calculation gives
\[
  \Pi_{\Rad}(g_N^v)(x)
  =
  \sum_{n=0}^N\sum_{|\alpha|=n}
  \frac{\ip{g}{h_n}_{L^2(\gamma)}}{\alpha!}
  \E[V^\alpha]\,H_\alpha(x).
\]

To justify the convergence in $L^2(\gamma^d)$ of the right-hand side, we note that
\[
  \norm{\Pi_{\Rad}(g_N^v)-\Pi_{\Rad}(g^v)}_{L^2(\gamma^d)}
  \le
  \norm{g_N^v-g^v}_{L^2(\gamma^d)}
  =
  \norm{g_N-g}_{L^2(\gamma)}
  \to 0,
\]
proving the statement.
\end{proof}

\begin{lemma}[{\cite[p.~447]{Folland2001}}]
\label{lem:spherical-moments}
Let $V$ be a random variable distributed uniformly on $\Sd$ and let $\alpha\in\N^d$ with $|\alpha|=n$.
If some $\alpha_i$ is odd, then
\[
  \E[V^\alpha]=0.
\]
If all $\alpha_i$ are even, then
\[
  \E[V^\alpha]
  =
  \frac{\displaystyle\frac{\alpha!}{(\alpha/2)!}}
  {\displaystyle
  \sum_{\substack{|\beta|=n\\ \beta\ {\rm even}}}
  \frac{(n/2)!}{(\beta/2)!}
  \frac{\beta!}{(\beta/2)!}}.
\]
\end{lemma}

The next proposition shows that every spherical moment of degree $n$ is
bounded in absolute value by the pure coordinate moment $\E[V_1^n]$, which
we compute explicitly.

\begin{proposition}
\label{prop:max-moment}
Assume $d\ge 1$. Let $\nu$ denote normalized surface measure on $\Sd$,
and let $V=(V_1,\ldots,V_d)$ be an $\Sd$-valued random vector with law
$\nu$. For even $n\ge2$, set
\[
  K_n:=\E[V_1^n]=\int_{\Sd}v_1^n\,d\nu(v).
\]
Then
\[
  |\E[V^\alpha]|\le K_n
\]
for every $\alpha\in\N^d$ with $|\alpha|=n$. Moreover,
\[
  K_n
  =
  \frac{(n-1)(n-3)\cdots 1}{d(d+2)\cdots(d+n-2)}
  =
  \prod_{i=0}^{n/2-1}\frac{2i+1}{d+2i}.
\]
\end{proposition}

\begin{proof}
Since $|\alpha|=n$, the generalized H\"older inequality, applied over the
indices for which $\alpha_i>0$ with exponents $n/\alpha_i$, and rotational
symmetry give
\[
  |\E[V^\alpha]|
  \le \prod_{\substack{1\le i\le d\\ \alpha_i>0}}
  \left(\int_{\Sd}|v_i|^n\,d\nu(v)\right)^{\alpha_i/n}
  =K_n^{|\alpha|/n}=K_n,
\]
proving the first statement. It remains to compute $K_n$.
If $d=1$, then $\Sd=\{\pm1\}$, so $K_n=1$, and every factor of the
product equals $(2i+1)/(1+2i)=1$. We may therefore assume $d\ge2$.

Since $v_1=\ip{v}{e_1}$, the degree-zero case of the Funk--Hecke formula
gives
\[
  K_n
  =
  \int_{\Sd}\ip{v}{e_1}^n\,d\nu(v)
  =
  c_d\int_{-1}^1 t^n(1-t^2)^{(d-3)/2}\,dt,
\]
where
\[
  c_d
  =
  \frac{\Gamma(d/2)}{\sqrt{\pi}\,\Gamma((d-1)/2)}.
\]
Since $n$ is even, the change of variables $u=t^2$ yields
\[
  \int_{-1}^1 t^n(1-t^2)^{(d-3)/2}\,dt
  =
  \int_0^1 u^{(n-1)/2}(1-u)^{(d-3)/2}\,du
  =
  \frac{\Gamma((n+1)/2)\Gamma((d-1)/2)}
       {\Gamma((d+n)/2)}.
\]
Consequently,
\[
  K_n
  =
  \frac{\Gamma(d/2)\Gamma((n+1)/2)}
       {\sqrt{\pi}\,\Gamma((d+n)/2)}.
\]
For even $n$, the gamma recurrence gives
\[
  \frac{\Gamma((n+1)/2)}{\sqrt{\pi}}
  =
  \frac{(n-1)(n-3)\cdots1}{2^{n/2}},
  \qquad
  \frac{\Gamma(d/2)}{\Gamma((d+n)/2)}
  =
  \frac{2^{n/2}}{d(d+2)\cdots(d+n-2)}.
\]
Multiplying these identities proves the desired product formula for $K_n$.
\end{proof}

\begin{remark}
\label{rem:moment-bounds}
The product formula also gives, when $n\ge2$ is even and $d\ge n$,
\[
  K_n\le \left(\frac{n}{d+n}\right)^{n/2}.
\]
Indeed, each factor satisfies
\[
  \frac{2i+1}{d+2i}\le \frac{n}{d+n}
\]
for $0\le i\le n/2-1$ under this assumption. Thus the size of the degree-$n$ radial component is governed by $K_n$, which
becomes exponentially small when the degree is proportional to the dimension.
Moreover, for every $d\ge1$ and even $n\ge2$,
\[
  \frac{K_{n+2}}{K_n}=\frac{n+1}{d+n}\le1,
\]
so $K_n$ is nonincreasing along even degrees.
\end{remark}

\section{Correlation Estimates and the Ridge Lower Bound}
\label{sec:ridge-lower-bound}

To control the convergence of the ridge approximations, we first bound the
correlation between the projections of a radial function and a ridge
function.
This is the Gaussian analogue of the correlation estimate in
Daniely's proof~\cite{Daniely2017}.

\begin{proposition}
\label{prop:chaos-correlation}
Let $n\in\N$. For every radial $F\in L^2(\gamma^d)$, every
$g\in L^2(\gamma)$, and every $v\in\Sd$, the inner product
\[
  \ip{P_nF}{P_n(g^v)}_{L^2(\gamma^d)}
\]
vanishes when $n$ is odd. When $n$ is even, it follows that
\[
  \left|\ip{P_nF}{P_n(g^v)}_{L^2(\gamma^d)}\right|
  \le
  \sqrt{K_n}\,
  \norm{P_nF}_{L^2(\gamma^d)}
  \norm{P_n(g^v)}_{L^2(\gamma^d)}.
\]
\end{proposition}

\begin{proof}
Since $F$ is radial, $P_nF$ is radial. Hence
\begin{align*}
  \ip{P_nF}{P_n(g^v)}
  &=\ip{P_nF}{\Pi_{\Rad}P_n(g^v)}\\
  &=\ip{P_nF}{P_n\Pi_{\Rad}(g^v)},
\end{align*}
where the second equality follows from
Proposition~\ref{prop:radialization-commutes}. By
Proposition~\ref{prop:radialized-ridge},
\[
  P_n\Pi_{\Rad}(g^v)(x)
  =
  \sum_{|\alpha|=n}
  \frac{\ip{g}{h_n}_{L^2(\gamma)}}{\alpha!}
  \E[V^\alpha]H_\alpha(x),
\]
where $V$ is uniformly distributed on $\Sd$.

If $n$ is odd, every $\alpha\in\N^d$ with $|\alpha|=n$ has an odd
component. Lemma~\ref{lem:spherical-moments} therefore gives
$\E[V^\alpha]=0$ for every term in the sum. Thus
$P_n\Pi_{\Rad}(g^v)=0$, and consequently
\[
  \ip{P_nF}{P_n(g^v)}=0.
\]

Assume now that $n$ is even. By Cauchy-Schwarz,
\begin{align}
  \label{eq:cs}
  \left|\ip{P_nF}{P_n(g^v)}\right|
  \le
  \norm{P_nF}\,\norm{\Pi_{\Rad}P_n(g^v)}.
\end{align}

Set $a_n=\ip{g}{h_n}_{L^2(\gamma)}$. By
Lemma~\ref{lem:hermite-ridge}, the degree-$n$ component of $g^v$ is
\[
  P_n(g^v)(x)
  =\frac{a_n}{n!}h_n(\ip{v}{x})
  =a_n\sum_{|\alpha|=n}\frac{v^\alpha}{\alpha!}H_\alpha(x).
\]
Using the two expansions above, we get
\begin{align*}
  \norm{\Pi_{\Rad}P_n(g^v)}^2
  &=\ip{P_n(g^v)}{\Pi_{\Rad}P_n(g^v)}\\
  &=a_n^2\sum_{|\alpha|=n}
    \frac{v^\alpha\E[V^\alpha]}{\alpha!}\\
  &=\frac{a_n^2}{n!}\E[\ip{v}{V}^n]
  =K_n\norm{P_n(g^v)}^2,
\end{align*}
which, together with \eqref{eq:cs}, proves the claim.
\end{proof}

We first derive from the preceding proposition the tail estimate needed to
prove the main lower bound.
For $n_0\in\N$, recall that we write
\[
  P_{\ge n_0}F=\sum_{n\ge n_0}P_nF.
\]

\begin{corollary}
\label{cor:tail-estimate}
Let $n_0$ be even. Then, for any
radial $F\in L^2(\gamma^d)$ and any
$g_1,\ldots,g_r\in L^2(\gamma)$ and $v_1,\ldots,v_r\in\Sd$,
\[
\left\|F-\sum_{j=1}^r g_j^{v_j}\right\|_{L^2(\gamma^d)}^2
\ge
\norm{P_{\ge n_0}F}^2
-
2\sqrt{K_{n_0}}\,
\norm{P_{\ge n_0}F}
\sum_{j=1}^r\norm{g_j}_{L^2(\gamma)}.
\]
\end{corollary}

\begin{proof}
Note that it holds that
\begin{align*}
  \left\|F-\sum_{j=1}^r g_j^{v_j}\right\|^2
& \ge
\left\|P_{\ge n_0}F-\sum_{j=1}^r P_{\ge n_0}(g_j^{v_j})\right\|^2\\
& \ge
\norm{P_{\ge n_0}F}^2
-
2\sum_{j=1}^r
\left|
\ip{P_{\ge n_0}F}{P_{\ge n_0}(g_j^{v_j})}
\right|
\end{align*}

Using orthogonality and
Proposition~\ref{prop:chaos-correlation} we get
\[
\left|
\ip{P_{\ge n_0}F}{P_{\ge n_0}(g_j^{v_j})}
\right|
\le
\sum_{\substack{n\ge n_0\\ n\ {\rm even}}}
\sqrt{K_n}\,\norm{P_nF}\,\norm{P_n(g_j^{v_j})}.
\]
Since $K_n\le K_{n_0}$ for $n\ge n_0$ along even degrees
(see Remark~\ref{rem:moment-bounds}),
Cauchy-Schwarz gives
\[
\le
\sqrt{K_{n_0}}\,
\left(\sum_{n\ge n_0}\norm{P_nF}^2\right)^{1/2}
\left(\sum_{n\ge n_0}\norm{P_n(g_j^{v_j})}^2\right)^{1/2}.
\]
The last factor is at most $\norm{g_j}_{L^2(\gamma)}$, so the claim follows.
\end{proof}

The tail estimate immediately proves the main lower bound.

\begin{proof}[Proof of Theorem~\ref{thm:ridge-lb}]
Recall that
$A=\norm{P_{\ge n_0}F}_{L^2(\gamma^d)}$ and that $M$ is such that
$\norm{g_j}_{L^2(\gamma)}\le M$ for every $1\le j\le r$.
If $A=0$, the conclusion is simply $r\ge0$ and is immediate.
We therefore assume $A>0$. Suppose that
\[
  \left\|F-\sum_{j=1}^r g_j^{v_j}\right\|_{L^2(\gamma^d)}^2
  \le \frac{A^2}{2}.
\]
The corollary implies that
\[
  \frac{A^2}{2}
  \ge
  A^2-2\sqrt{K_{n_0}}\,A\,rM.
\]
Rearranging the terms yields
\[
  r\ge \frac{A}{4M\sqrt{K_{n_0}}}.
\]
\end{proof}

\section{Depth Separation}
\label{sec:depth-separation}

We now turn to the proof of Theorem~\ref{thm:depth-sep}. We consider a rapidly
oscillating function of $\norm{x}^2$ that has substantial weight in the
high-chaos energy but is still well approximated by a depth-3 ReLU neural
network.

Following the oscillatory constructions in
\cite{Daniely2017}, for parameters $d,m>0$, define the
radial function
\[
  F_m(x)=\sin\!\left(\pi m\,\frac{\|x\|^2}{d}\right), \qquad x\in\R^d.
\]
Since $|F_m|\le 1$, this is well-defined in $L^2(\gamma^d)$.  Set
\[
  f_m(t)=\sin(\pi mt), \qquad t\ge0,
\]
and denote by
\[
  \rho_d(t)
  =\frac{(d/2)^{d/2}}{\Gamma(d/2)}
    t^{d/2-1}e^{-dt/2},
  \qquad t>0,
\]
the density of $\norm{X}^2/d$ for $X\sim\gamma^d$.  We first prove two
preliminary facts.

\begin{lemma}
\label{lem:radial-poly-distance}
Let $d\ge1$ and let $F\in L^2(\gamma^d)$ be of the form
$F(x)=\phi(\norm{x}^2/d)$ for a measurable function
$\phi\colon[0,\infty)\to\R$. For every even integer $n_0\ge2$,
\[
  \norm{P_{\ge n_0}F}^2_{L^2(\gamma^d)}
  =\inf_{\substack{q\in\R[t]\\\deg q<n_0/2}}
    \int_0^\infty(\phi(t)-q(t))^2\rho_d(t)\,dt.
\]
\end{lemma}

\begin{proof}
If $\mathcal R$ denotes the subspace of radial functions and
$\mathcal H_{<n_0}$ the direct sum of the Gaussian chaoses of
degree less than $n_0$, then
\[
  \norm{P_{\ge n_0}F}_{L^2(\gamma^d)}
  =\inf_{G\in\mathcal H_{<n_0}\cap\mathcal R}
    \norm{F-G}_{L^2(\gamma^d)},
\]
since the orthogonal projection $P_{<n_0}F$ is the function in
$\mathcal H_{<n_0}$ that minimizes the distance to $F$. We can restrict the infimum over the radial functions,
since Proposition~\ref{prop:radialization-commutes} implies that
$P_{<n_0}F$ is radial as well.

Every function in $\mathcal H_{<n_0}\cap\mathcal R$ is a radial polynomial of
degree less than $n_0$, and hence has the form
$q(\norm{x}^2/d)$ with $\deg q<n_0/2$.  Conversely, every polynomial of this
form belongs to $\mathcal H_{<n_0}\cap\mathcal R$.  
Since, for a random variable 
$X\sim\gamma^d$, 
$\norm{X}^2/d$ has density $\rho_d$ and $F(X)=\phi(\norm{X}^2/d)$, we get
the desired result.
\end{proof}

\begin{lemma}
\label{lem:chi-square-local-density}
There exist universal constants $c_1,d_0>0$ such that, for every $d\ge d_0$,
the density $\rho_d$ satisfies
\[
  \rho_d(t)\ge c_1\sqrt d
\]
for every $t\in[1-1/\sqrt d,\,1+1/\sqrt d]$.
\end{lemma}

\begin{proof}
Note that, by Stirling's formula,
\[
  \rho_d(1)
  =
  \frac{(d/2)^{d/2}e^{-d/2}}{\Gamma(d/2)}
  \sim \sqrt{\frac{d}{4\pi}}.
\]
Thus, there exists $d_0$ such that, for every $d\ge d_0$,
\[
  \rho_d(1)\ge \frac12\sqrt{\frac{d}{4\pi}}.
\]
If $|t-1|\le1/\sqrt d$, then $t\ge1/2$ for $d\ge4$. Using
\[
  (t-1)-(t-1)^2\le\log t\le t-1,
\]
we observe that
\begin{align*}
  \log\frac{\rho_d(t)}{\rho_d(1)}
  &=\left(\frac d2-1\right)\log t-\frac d2(t-1)\\
  &=\frac d2\bigl(\log t-(t-1)\bigr)-\log t\\
  &\ge-\frac d2(t-1)^2-(t-1)\\
  &\ge-\frac12-\frac1{\sqrt d}
  \ge-1.
\end{align*}
Putting everything together, there exists a universal constant $c_1>0$ such
that
\[
  \rho_d(t)\ge c_1\sqrt{d}
\]
for all $t\in[1-1/\sqrt{d},\,1+1/\sqrt{d}]$ and $d\ge d_0$.
\end{proof}

We now apply these two facts to the radial sine function.  The key ingredient
is the following lemma.

\begin{lemma}
\label{lem:energy-lb}
There exist universal constants $c,d_0>0$ such that for all $d\ge d_0$, all
even integers $n_0\ge 2$, and all $m\ge 4n_0\sqrt{d}$,
\[
  \norm{P_{\ge n_0}F_m}_{L^2(\gamma^d)}\ge c.
\]
\end{lemma}

\begin{proof}
By Lemma~\ref{lem:radial-poly-distance},
\[
  \norm{P_{\ge n_0}F_m}^2_{L^2(\gamma^d)}
  =\inf_{\substack{q\in\R[t]\\\deg q<n_0/2}}
    \int_0^\infty(f_m(t)-q(t))^2\rho_d(t)\,dt.
\]

Fix an arbitrary polynomial $q$ of degree less than $n_0/2$. By
Lemma~\ref{lem:chi-square-local-density}, we can restrict the integral to an
interval around $1$ on which $\rho_d$ is bounded below. Therefore,
\begin{equation}
\label{eq:density-lb}
  \int_0^\infty(f_m(t)-q(t))^2\rho_d(t)\,dt
  \;\ge\;
  c_1\sqrt{d}
  \int_{1-1/\sqrt{d}}^{1+1/\sqrt{d}}(f_m(t)-q(t))^2\,dt.
\end{equation}

We next rescale this interval and compare the signs of $f_m$ and $q$.
Substituting $t=1+u/(\pi m)$ gives
\begin{equation}
\label{eq:cov}
\begin{aligned}
  &\int_{1-1/\sqrt{d}}^{1+1/\sqrt{d}}(f_m(t)-q(t))^2\,dt\\
  &\qquad=
  \frac{1}{\pi m}\int_{-\pi m/\sqrt d}^{\pi m/\sqrt d}
  \left(\sin(\pi m+u)-q\!\left(1+\frac{u}{\pi m}\right)\right)^2\,du.
\end{aligned}
\end{equation}

The zeros of $\sin(\pi m+u)$ are spaced $\pi$ apart.
Hence $[-\pi m/\sqrt d,\,\pi m/\sqrt d]$ contains
\[
  N\ge 2\left\lfloor\frac{m}{\sqrt d}\right\rfloor-2
\]
full intervals between consecutive zeros of $\sin(\pi m+u)$; denote these half-period
intervals by $K_1,\ldots,K_N$.  Each $K_j$ has length $\pi$, and the sign of
$\sin(\pi m+u)$ alternates from $K_j$ to $K_{j+1}$.

If $q\equiv0$, the desired estimate follows by summing
$\int_{K_j}\sin^2(\pi m+u)\,du=\pi/2$ over the intervals $K_j$.
Suppose, therefore, that $q$ is nonzero. As a polynomial in $u$,
$q(1+u/(\pi m))$ also has degree less than $n_0/2$, so it has at most
$n_0/2-1$ real zeros. Discard every
$K_j$ whose interior contains one of these zeros.  At most $n_0/2-1$ intervals
are discarded, so the number $N'$ of remaining intervals satisfies
\[
  N'\ge N-\frac{n_0}{2}+1.
\]

Arrange the remaining $K_j$ into maximal blocks
$R_1,\ldots,R_G$ of intervals that are consecutive in the original
list $K_1,\ldots,K_N$ (with no discarded interval between them) and on whose
interiors $q(1+u/(\pi m))$ has the same sign.  For each block, define
\[
  s_g:=\#\{K_j:K_j\text{ belongs to }R_g\},
\]
and note that
\[
  \sum_{g=1}^G s_g=N'.
\]
Figure~\ref{fig:interval-blocks} illustrates this decomposition.

\begin{figure}[tbp]
\centering
\includegraphics[width=\linewidth]{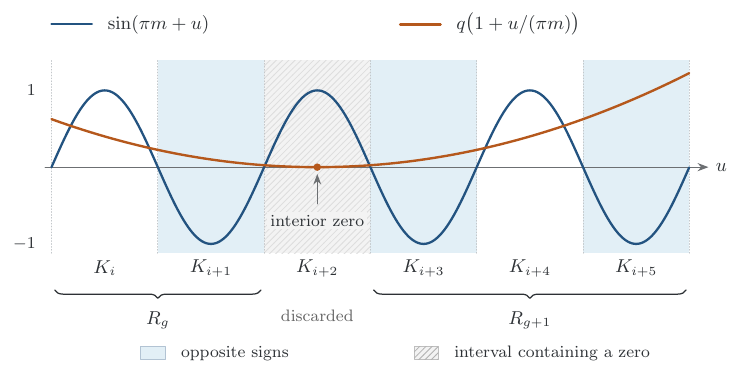}
\footnotesize\setstretch{1}
\caption{Two blocks separated by the discarded $K_{i+2}$.
Shading indicates opposite signs.}
\label{fig:interval-blocks}
\end{figure}

Each separation
between two successive blocks requires a zero of $q(1+u/(\pi m))$: either
the blocks are separated by a discarded interval, whose interior contains a zero
of $q(1+u/(\pi m))$, or they are adjacent and $q(1+u/(\pi m))$ changes sign at their common
endpoint. The closed gaps between successive blocks are pairwise disjoint, so
distinct separations use distinct zeros.
Since $q(1+u/(\pi m))$ has at
most $n_0/2-1$ zeros, note that the bound
\[
  G\le\frac{n_0}{2}
\]
holds.

Fix $R_g$ and note that 
the sign of $q(1+u/(\pi m))$ is the same on all of its $s_g$
intervals, whereas the sign of $\sin(\pi m+u)$ alternates.  
Therefore $\sin(\pi m+u)$ and $q(1+u/(\pi m))$ have opposite signs on at least
\[
  \left\lfloor\frac{s_g}{2}\right\rfloor
  \ge \frac{s_g-1}{2}
\]
intervals of $R_g$. Consequently, the total number $M_{\mathrm{opp}}$ of
opposite-sign intervals satisfies
\[
  M_{\mathrm{opp}}
  \ge \sum_{g=1}^G\frac{s_g-1}{2}
  =\frac{N'-G}{2}
  \ge\frac{N'-n_0/2}{2}
  \;\ge\;
  \frac{N-n_0}{2}.
\]
Since $m\ge 4n_0\sqrt{d}$,
we conclude that 
$(N-n_0)/2\ge N/4$, and
$N/4\ge m/(2\sqrt{d})-1\ge m/(4\sqrt{d})$.
It follows that $M_{\mathrm{opp}}\ge m/(4\sqrt{d})$.

On any interval $K_j$ where $\sin(\pi m+u)$ and $q(1+u/(\pi m))$ have opposite signs on the
interior, we have
\[
\begin{aligned}
  &\int_{K_j}\left(\sin(\pi m+u)-q\!\left(1+\frac{u}{\pi m}\right)\right)^2\,du\\
  &\qquad\ge\int_{K_j}\sin^2(\pi m+u)\,du=\frac{\pi}{2}.
\end{aligned}
\]
Therefore
\begin{equation}
\label{eq:osc-lb}
\begin{aligned}
  &\int_{-\pi m/\sqrt d}^{\pi m/\sqrt d}
  \left(\sin(\pi m+u)-q\!\left(1+\frac{u}{\pi m}\right)\right)^2\,du\\
  &\qquad\ge
  \frac{m}{4\sqrt{d}}\cdot\frac{\pi}{2}
  =
  \frac{\pi m}{8\sqrt{d}}.
\end{aligned}
\end{equation}

Combining \eqref{eq:density-lb}, \eqref{eq:cov}, \eqref{eq:osc-lb}:
\[
  \int_0^\infty(f_m(t)-q(t))^2\rho_d(t)\,dt
  \;\ge\;
  c_1\sqrt{d}\cdot\frac{1}{\pi m}\cdot\frac{\pi m}{8\sqrt{d}}
  =
  \frac{c_1}{8}.
\]
Since $q$ was arbitrary,
$\norm{P_{\ge n_0}F_m}^2_{L^2(\gamma^d)}\ge c_1/8$.
Taking $c=\sqrt{c_1/8}>0$ proves the lemma.
\end{proof}

Now we are ready to prove Theorem~\ref{thm:depth-sep}.

\begin{proof}[Proof of Theorem~\ref{thm:depth-sep}]

\textup{(i)} Set $n_0=d$ and $m=4d^{3/2}$. Since $d$ is even, $n_0$ is an
even integer $\ge 2$, and
$m=4d^{3/2}=4d\cdot\sqrt{d}=4n_0\sqrt{d}$, so the hypotheses of
Lemma~\ref{lem:energy-lb} are satisfied. Moreover,
\[
  F_m(x)
  =\sin\!\left(\pi(4d^{3/2})\frac{\norm{x}^2}{d}\right)
  =\sin\!\left(4\pi d^{1/2}\norm{x}^2\right)
  =F(x).
\]
Therefore, Lemma~\ref{lem:energy-lb} gives
$A:=\norm{P_{\ge d}F}_{L^2(\gamma^d)}\ge c$. By Remark~\ref{rem:moment-bounds},
 $K_n\le(n/(d+n))^{n/2}$ whenever $d\ge n$.  Taking $n=d$,
\[
  K_d
  \le
  \left(\frac{d}{d+d}\right)^{d/2}
  =
  \left(\frac12\right)^{d/2}
  =
  2^{-d/2},
\]
so $1/\sqrt{K_d}\ge 2^{d/4}=2^{\Omega(d)}$.
Theorem~\ref{thm:ridge-lb} then gives
\[
  r\ge \frac{A}{4M\sqrt{K_d}}
  \ge \frac{c}{4M}\,2^{d/4}.
\]

\medskip
\textup{(ii)} The construction follows the strategy in
\cite{Daniely2017}. Fix $0<\epsilon\le1$ and set
$\eta=\epsilon/(16\pi d^{1/2})$, so that $0<\eta\le1$.
For every $R>0$, we first construct a globally bounded depth-3 network
$\mathcal N_{R,\epsilon}$ satisfying
\[
  \sup_{\norm{x}\le R}|F(x)-\mathcal N_{R,\epsilon}(x)|\le\frac{\epsilon}{2}.
\]

First, we construct a depth-2 network that approximates $\norm{x}^2$.
As in Lemma~5 of \cite{Daniely2017}, we use piecewise-linear interpolation.
We write the interpolants as linear combinations of
$\sigma(t-a)-\sigma(t-b)$, so that they are constant outside the interpolation
interval. Let $K\in\N$ and $t_j=jR/K$ for $0\le j\le K$. The function
\[
  p(t)=\sum_{j=0}^{K-1}(t_j+t_{j+1})\bigl(\sigma(t-t_j)-\sigma(t-t_{j+1})\bigr)
\]
is the piecewise-linear interpolant of $t\mapsto t^2$ at the nodes
$t_0,\ldots,t_K$, since
$p(t)=t_j^2+(t_j+t_{j+1})(t-t_j)$ on $[t_j,t_{j+1}]$;
it vanishes on $(-\infty,0]$ and is constant on
$[R,\infty)$.
For $t\in[t_j,t_{j+1}]$, the preceding formula gives
\[
  0\le p(t)-t^2=(t-t_j)(t_{j+1}-t)
  \le\frac{R^2}{4K^2}\le\frac{4R^2}{K}.
\]
As
$p(t)+p(-t)=p(|t|)$, it follows that $|p(t)+p(-t)-t^2|\le4R^2/K$ for
$|t|\le R$. Set
\[
  \mathcal N_{\mathrm{sq}}(x)=\sum_{i=1}^d\bigl(p(x_i)+p(-x_i)\bigr),
  \qquad
  K=\left\lceil\frac{4dR^2}{\eta}\right\rceil.
\]
If $\norm{x}\le R$, then $|x_i|\le R$ for every $i$, and therefore
\[
  \left|\mathcal N_{\mathrm{sq}}(x)-\|x\|^2\right|
  \le\eta.
\]
Thus $\mathcal N_{\mathrm{sq}}$ is a depth-2 network with $4dK$ hidden units of
the form $\sigma(\pm x_i-t_j)$, that is, with weight vectors $\pm e_i$, biases
$-t_j\in[-R,0]$, and output coefficients $\pm(t_j+t_{j+1})$, bounded by $2R$.

Next, using the same approach, we approximate the sine function.
Let $K'\in\N$ and $u_l=-1+l(R^2+2)/K'$ for
$0\le l\le K'$, so that $u_0=-1$ and $u_{K'}=R^2+1$. Define
\[
  \lambda_l=\frac{\sin(4\pi d^{1/2}u_{l+1})-\sin(4\pi d^{1/2}u_l)}{u_{l+1}-u_l},
  \qquad 0\le l<K'.
\]
The function
\[
  \mathcal N_{\sin}(s)=\sin(4\pi d^{1/2}u_0)
  +\sum_{l=0}^{K'-1}\lambda_l\bigl(\sigma(s-u_l)-\sigma(s-u_{l+1})\bigr)
\]
is the piecewise-linear interpolant of $s\mapsto\sin(4\pi d^{1/2}s)$ at the
nodes $u_0,\ldots,u_{K'}$. Indeed, on $[u_l,u_{l+1}]$, the preceding summands
telescope, giving
\[
  \mathcal N_{\sin}(s)=\sin(4\pi d^{1/2}u_l)+\lambda_l(s-u_l).
\]
It is constant on $(-\infty,-1]$ and on $[R^2+1,\infty)$. On each grid
interval its values are convex combinations of the endpoint sine values,
so $|\mathcal N_{\sin}(s)|\le1$ for all $s\in\R$.
Since $|\sin(4\pi d^{1/2}s)-\sin(4\pi d^{1/2}t)|\le4\pi d^{1/2}|s-t|$, we have
$|\lambda_l|\le4\pi d^{1/2}$. Hence, for $s\in[u_l,u_{l+1}]$,
\[
  \begin{aligned}
  &|\mathcal N_{\sin}(s)-\sin(4\pi d^{1/2}s)|\\
  &\qquad\le |\sin(4\pi d^{1/2}u_l)-\sin(4\pi d^{1/2}s)|+|\lambda_l|(s-u_l)\\
  &\qquad\le 8\pi d^{1/2}(s-u_l)
  \le\frac{8\pi d^{1/2}(R^2+2)}{K'}.
  \end{aligned}
\]
Taking
\[
  K'=\left\lceil\frac{32\pi d^{1/2}(R^2+2)}{\epsilon}\right\rceil
\]
gives
\[
  \sup_{s\in[-1,R^2+1]}
  \left|\mathcal N_{\sin}(s)-\sin(4\pi d^{1/2}s)\right|
  \le\frac{\epsilon}{4}.
\]
Thus $\mathcal N_{\sin}$ is a depth-2 network with $2K'+1$ hidden units:
the units $\sigma(s-u_l)$ and $\sigma(s-u_{l+1})$, with output coefficients
$\lambda_l$ and $-\lambda_l$, and one constant unit $\sigma(1)=1$, with
output coefficient $\sin(4\pi d^{1/2}u_0)$. The input weights are $1$ or $0$,
the biases have absolute value at most $R^2+1$, and the output coefficients
have absolute value at most $4\pi d^{1/2}$.
Since $\eta\le1$, the interval
$[-1,R^2+1]$ contains $\mathcal N_{\mathrm{sq}}(x)$ whenever $\norm{x}\le R$.

Finally, composing these two networks, set
$\mathcal N_{R,\epsilon}:=\mathcal N_{\sin}\circ\mathcal N_{\mathrm{sq}}$. This is
a depth-3 network. For $\|x\|\le R$,
\[
  \begin{aligned}
  \left|\mathcal N_{R,\epsilon}(x)-F(x)\right|
  &\le
  \left|\mathcal N_{\sin}(\mathcal N_{\mathrm{sq}}(x))
    -\sin\bigl(4\pi d^{1/2}\mathcal N_{\mathrm{sq}}(x)\bigr)\right| \\
  &\quad
  +4\pi d^{1/2}\left|\mathcal N_{\mathrm{sq}}(x)-\|x\|^2\right| \\
  &\le
  \frac{\epsilon}{4}+4\pi d^{1/2}\eta=\frac{\epsilon}{2}.
  \end{aligned}
\]
Then $\mathcal N_{\mathrm{sq}}$ has width
$4dK=O(d^2R^2/\eta+d)=O(d^{5/2}(R^2+1)/\epsilon)$ and
$\mathcal N_{\sin}$ has width
$2K'+1=O(d^{1/2}(R^2+1)/\epsilon)$, so the width of
$\mathcal N_{R,\epsilon}$ is $O(d^{5/2}(R^2+1)/\epsilon)$.
For the parameters: the first-layer weight vectors $\pm e_i$ have
norm $1$ and the first-layer biases satisfy $|t_j|\le R$; the second-layer
weights satisfy $|t_j+t_{j+1}|\le2R$, and the second-layer biases are $-u_l$,
with $|u_l|\le R^2+1$, or equal to $1$; the output weights satisfy
$|\lambda_l|\le4\pi d^{1/2}$ and $|\sin(4\pi d^{1/2} u_0)|\le1$. Since $R\le R^2+1$, all
parameters are bounded by $2(R^2+1)+4\pi d^{1/2}=O(d^{1/2}+R^2)$.

It remains to choose $R$. If $X\sim\gamma^d$, then
\[
  \Pr(\norm{X}>R)
  \le
  \E\!\left[e^{\norm{X}^2/4}\right]e^{-R^2/4}
  =
  2^{d/2}e^{-R^2/4}.
\]
Set
\[
  R^2=2d\log 2+8\log(4/\epsilon).
\]
Then $\Pr(\norm{X}>R)\le\epsilon^2/16$. Since both $F$ and
$\mathcal N_{R,\epsilon}$ are bounded in absolute value by $1$,
\[
  \norm{F-\mathcal N_{R,\epsilon}}_{L^2(\gamma^d)}^2
  \le
  \frac{\epsilon^2}{4}+4\Pr(\norm{X}>R)
  \le
  \frac{\epsilon^2}{2}
  \le
  \epsilon^2.
\]
This choice gives $R^2=O(d+\log(1/\epsilon))$ and hence the claimed bounds.
\end{proof}

\section{Lean formalization}
\label{sec:lean-verification}

\begingroup
All theorems, lemmas, propositions, and corollaries have been formalized and verified
in Lean~4 using Mathlib. The Lean statements of the theorems, lemmas, and
corollaries have also been audited by the authors.
The formalization was developed with assistance from the large
language models (LLMs) Claude Fable (Anthropic) and GPT Sol (OpenAI).
The complete source code is available in the companion GitHub
repository.\footnote{\raggedright\url{https://github.com/cynthiabort/Depth-Separation-for-Neural-Networks-under-Gaussian-Measure---Lean-}\par}

Running \texttt{lake build} checks the mathematical library and the
standalone axiom-audit regression tests. The audit accepts only the standard
axioms \texttt{propext}, \texttt{Classical.choice}, and \texttt{Quot.sound};
neither admitted proofs nor project-defined axioms are used in the
mathematical library.

Below we reproduce the Lean statements of Theorems~\ref{thm:ridge-lb}
and~\ref{thm:depth-sep}, omitting their proofs. For
Theorem~\ref{thm:depth-sep}, the displayed declarations give its two parts;
the declaration \texttt{depth\_separation} combines them into a single
statement with universal constants.

All declarations belong to the \texttt{GaussianDepthSeparation} namespace.
\par\endgroup

\subsection*{Theorem~\ref{thm:ridge-lb}: ridge lower bound}
\begingroup
Let $d\ge1$, $n_0\in\N$ be even, $F\in L^2(\gamma^d)$ be radial,
$M>0$, and $g_j\in L^2(\gamma)$, $v_j\in\Sd$ with $\norm{g_j}_2\le M$.
Set $S=\sum_{j=1}^r g_j(\ip{v_j}{\cdot})$ and
$A=\norm{P_{\ge n_0}F}_2$. Then
\[
  \norm{F-S}_2^2\le A^2/2
  \quad\Longrightarrow\quad
  r\ge\frac{A}{4M\sqrt{K_{n_0}}}.
\]
\par\endgroup
\noindent\begin{minipage}{\linewidth}
\begin{lstlisting}[style=leanstatement]
theorem ridge_lower_bound {d : ℕ} (hd : 0 < d) {n₀ : ℕ} (hn₀ : Even n₀)
    (F : L2 d) (hF : IsRadial F) {r : ℕ} (g : Fin r → L2R) (v : Fin r → Vec d)
    (hv : ∀ j, ‖v j‖ = 1) {M : ℝ} (hM : 0 < M) (hgM : ∀ j, ‖g j‖ ≤ M)
    (happrox : ‖F - ∑ j, ridge (v j) (hv j) (g j)‖ ^ 2 ≤ ‖Pge d n₀ F‖ ^ 2 / 2) :
    ‖Pge d n₀ F‖ / (4 * M * Real.sqrt (Kmom d n₀)) ≤ (r : ℝ)
\end{lstlisting}
\nopagebreak[4]
\noindent{\footnotesize Source:
\nolinkurl{GaussianDepthSeparation/LowerBound.lean}, declaration \texttt{ridge\_lower\_bound}.\par}
\end{minipage}\par

\subsection*{Theorem~\ref{thm:depth-sep}: depth separation}
\begingroup
Set $F(x)=\sin(4\pi\sqrt d\,\norm{x}^2)$.
\textup{(i)} For even $d\ge d_0:=\texttt{chiSquareThreshold}$ (a fixed
threshold, at least $16$, chosen from Stirling's formula as in
Lemma~\ref{lem:chi-square-local-density}), $c=\texttt{cEnergy}>0$ and
$M,g_j,v_j,S$ as above,
\[
  \norm{F-S}_2^2\le c^2/2
  \quad\Longrightarrow\quad
  r\ge\frac{c\,2^{d/4}}{4M}.
\]
\par\endgroup
\noindent\begin{minipage}{\linewidth}
\begin{lstlisting}[style=leanstatement]
theorem depth_separation_lower (d : ℕ) (hde : Even d) (hd : chiSquareThreshold ≤ d)
    {r : ℕ} (g : Fin r → L2R) (v : Fin r → Vec d) (hv : ∀ j, ‖v j‖ = 1)
    {M : ℝ} (hM : 0 < M) (hgM : ∀ j, ‖g j‖ ≤ M)
    (happrox : ‖sepTarget d - ∑ j, ridge (v j) (hv j) (g j)‖ ^ 2 ≤
      cEnergy ^ 2 / 2) :
    cEnergy * 2 ^ ((d : ℝ) / 4) / (4 * M) ≤ (r : ℝ)
\end{lstlisting}
\nopagebreak[4]
\noindent{\footnotesize Source:
\nolinkurl{GaussianDepthSeparation/DepthSeparation.lean}, declaration \texttt{depth\_separation\_lower}.\par}
\end{minipage}\par

\begingroup
\textup{(ii)} There exists a universal $C>0$ such that, for every $d\ge1$
and $0<\epsilon\le1$, there exists a depth-3 ReLU network $\mathcal N$ with
hidden-layer widths $w_1,w_2$ such that
\[
  \norm{F-\mathcal N}_2\le\epsilon,
  \qquad
  \max\{w_1,w_2\}\le\frac{C d^{5/2}(d+\log(1/\epsilon))}{\epsilon},
\]
with all parameters bounded by $C(d+\log(1/\epsilon))$.
\par\endgroup
\noindent\begin{minipage}{\linewidth}
\begin{lstlisting}[style=leanstatement]
theorem depth_separation_upper_L2 :
    ∃ C : ℝ, 0 < C ∧ ∀ d : ℕ, 0 < d → ∀ ε : ℝ, 0 < ε → ε ≤ 1 →
      ∃ N : Net3 d, ∃ hN : MemLp N.eval 2 (γ d),
        ‖sepTarget d - hN.toLp N.eval‖ ≤ ε ∧
        max (N.n1 : ℝ) (N.n2 : ℝ) ≤
          C * ((d : ℝ) ^ 2 * Real.sqrt d) * ((d : ℝ) + Real.log (1 / ε)) / ε ∧
        N.WeightsLe (C * ((d : ℝ) + Real.log (1 / ε)))
\end{lstlisting}
\nopagebreak[4]
\noindent{\footnotesize Source:
\nolinkurl{GaussianDepthSeparation/GlobalUpper.lean}, declaration \texttt{depth\_separation\_upper\_L2}.\par}
\end{minipage}\par

\fi

\ifshowfullpaper
\enlargethispage{2\baselineskip}
\else
\setbox\hiddenbibliography=\vbox\bgroup
\fi

\ifshowfullpaper\else
\egroup
\fi

\end{document}